\documentclass{article}
\usepackage{mathrsfs}
\usepackage{arxiv}
\usepackage{amsthm}
\usepackage[utf8]{inputenc}
\usepackage[T1]{fontenc}
\usepackage{hyperref}
\usepackage{url}
\usepackage{booktabs}
\usepackage{amsfonts}
\usepackage{nicefrac}
\usepackage{microtype}
\usepackage{lipsum}
\usepackage{graphicx}
\usepackage{amsmath}
\newtheorem{theorem}{Theorem}[section]   
\usepackage{amsmath}
\usepackage{algorithm}
\usepackage{algpseudocode}
\theoremstyle{remark}
\newtheorem{remark}{Remark}
\graphicspath{ {./images/} }

\title{Directional Kernel Mean Difference: A Fast Signed Statistic for Univariate Distribution Comparison}

\author{ Shijie Zhong \\
	School of Power and Energy\\
	Northwestern Polytechnical University\\
	Xi'an, Shanxi 710129 \\
	\texttt{zhongsj@mail.nwpu.edu.cn} \\
	\And Jiangfeng Fu \\
	School of Power and Energy\\
	Northwestern Polytechnical University\\
	Xi'an, Shanxi 710129 \\
	\texttt{fjf@nwpu.edu.cn} }

\begin{document}
	\maketitle
	\begin{abstract} We introduce the Directional Kernel Mean Difference (DKMD), a signed statistic for univariate distribution comparison that preserves the direction of distributional shifts. Unlike the squared Maximum Mean Discrepancy (MMD), which discards directional information by squaring the RKHS distance, DKMD integrates the difference of kernel mean embeddings against a fixed odd weighting function. This construction yields three structural properties: antisymmetry, immunity to symmetric distributional differences, and directional monotonicity under stochastic dominance. We derive a data-driven Riemann estimator that ensures asymptotic consistency with the continuous formulation, strictly preserving the theoretical guarantees of the signed statistic in empirical evaluations. To overcome the quadratic computational cost of kernel methods, we develop an $O(N \log N)$ prefix--suffix scanning algorithm that exploits the total order of the real line while requiring only $O(N)$ memory. Experiments on synthetic benchmarks demonstrate that DKMD correctly isolates directional shifts from symmetric perturbations, remains robust to heavy-tailed outliers that can flip the sign of the mean difference, and scales to millions of samples in seconds.
	\end{abstract}
	
	\section{Introduction}
	
	Comparing two probability distributions is a fundamental problem in statistics and machine learning, with applications spanning A/B testing, covariate shift detection, and generative model evaluation. In many practical scenarios, however, the key question is not merely whether two distributions differ, but also \emph{in which direction} the distribution has shifted. For example, in a large-scale A/B test of system latency, an increase from 100~ms to 105~ms indicates a directional degradation in service quality, whereas an increase in variability from 10~ms to 20~ms without a systematic change in location represents a qualitatively different phenomenon. A desirable statistic for distributional comparison should therefore detect directional shifts, remain insensitive to symmetric variations without directional meaning, and scale to large-scale modern datasets.
	
	The Maximum Mean Discrepancy (MMD; \cite{gretton2012kernel}) is among the most widely used kernel-based approaches for two-sample comparison. By embedding probability distributions into a reproducing kernel Hilbert space (RKHS), $\mathrm{MMD}^{2}$ measures the squared distance between their kernel mean embeddings. This formulation yields a powerful and general measure of distributional discrepancy, but the RKHS norm makes the statistic inherently non-negative and symmetric: it quantifies the magnitude of differences while discarding the direction of the underlying shift. Consequently, $\mathrm{MMD}^{2}$ cannot indicate whether one distribution has shifted toward larger or smaller values relative to another.
	
	We propose the \textbf{Directional Kernel Mean Difference (DKMD)}, a signed statistic that restores directional information in distributional comparison. Instead of measuring the norm of the embedding difference, $\mathrm{DKMD}$ integrates the difference of kernel mean embeddings against a fixed odd weighting function $w(z)$. This construction preserves the orientation of distributional changes while allowing the statistic to retain meaningful algebraic properties. In particular, $\mathrm{DKMD}$ is not designed as a distance metric; rather, it is an oriented discrepancy measure whose sign identifies the direction of distributional shift.
	
	Our contributions are summarized as follows:
	
	\begin{enumerate}
		\item \textbf{Directional distributional discrepancy.} 	We introduce $\mathrm{DKMD}$ and establish three theoretical properties: antisymmetry ($\mathrm{DKMD}(P,Q)=-\mathrm{DKMD}(Q,P)$), immunity to symmetric distributional differences, and directional monotonicity under stochastic dominance.
		
		\item \textbf{Consistent empirical estimation.} We develop a data-driven Riemann sum estimator that explicitly recovers the unweighted Lebesgue integral in the large-sample limit. This ensures that the empirical statistic asymptotically inherits the strict directional guarantees of the continuous formulation.
		
		\item \textbf{Scalable computation.} 	We develop an efficient prefix--suffix scanning algorithm that exploits the ordering structure of the real line. The proposed method reduces the computational complexity from quadratic time to $O(N\log N)$ with $O(N)$ memory consumption.
		
		\item \textbf{Empirical evaluation.} 	Through controlled synthetic experiments, we demonstrate that $\mathrm{DKMD}$ distinguishes directional shifts from symmetric perturbations, remains robust against heavy-tailed outliers, and 	scales to million-scale datasets.
	\end{enumerate}
	
	The remainder of the paper is organized as follows. Section~2 reviews kernel mean embeddings and the limitations of squared $\mathrm{MMD}$. Section~3 defines $\mathrm{DKMD}$ and establishes its theoretical properties. Section~4 presents the fast $O(N\log N)$ algorithm. Section~5 reports synthetic benchmarks. Section~6 discusses related work, and Section~7 concludes.
	
	\section{Preliminaries}
	
	\subsection{Notation}
	
	Let $P$ and $Q$ be Borel probability measures on $\mathbb{R}$. Given i.i.d.\ samples $X_1,\dots,X_n \sim P$ and $Y_1,\dots,Y_m \sim Q$, we wish to decide whether $P$ and $Q$ differ and, if so, in which direction. Denote the combined sample size by $N = n + m$.
	
	\subsection{Kernel mean embedding}
	
	Let $k_\sigma : \mathbb{R} \times \mathbb{R} \to \mathbb{R}_+$ be a positive-definite kernel with bandwidth $\sigma > 0$ \cite{Scholkopf2001LearningWK}. The \emph{kernel mean embedding} of $P$ into the associated RKHS $\mathcal{H}$ is the function
	\begin{equation}
		\mu_P(\cdot) = \mathbb{E}_{X \sim P}\,k_\sigma(\cdot, X),
	\end{equation} evaluated at a point $z$ as $\mu_P(z) = \mathbb{E}_{X \sim P}\,k_\sigma(z, X)$.
	
	When $k_\sigma$ is characteristic \cite{Sriperumbudur2009HilbertSE}, the map $P \mapsto \mu_P(\cdot)$ is injective, so $P \neq Q$ implies $\mu_P(\cdot) \neq \mu_Q(\cdot)$ as elements of $\mathcal{H}$.
	
	We adopt the Laplacian kernel
	\begin{equation} k_\sigma(x, z) = \exp\!\left(-\frac{|x - z|}{\sigma}\right)
	\end{equation} as our primary kernel; its exponential form is essential to the recursive algorithm of Section~4. The theoretical results in Section~3 are proved for translation-invariant characteristic kernels (Theorem~3 is specialized to the Laplacian kernel; see Appendix~A).
	
	\subsection{The squared MMD and its blind spot}
	
	The Maximum Mean Discrepancy \cite{gretton2012kernel} is defined as
	\begin{equation} \mathrm{MMD}^2(P, Q) = \left\| \mu_P - \mu_Q \right\|^2_{\mathcal{H}}.
	\end{equation}
	
	The squaring operation erases sign information entirely. Consider two scenarios: a location shift $P = N(0,1), Q = N(\mu,1)$ with $\mu \neq 0$, and a scale shift $P = N(0,1), Q = N(0,\sigma^2)$ with $\sigma^2 \neq 1$. For certain kernel and parameter choices, $\mathrm{MMD}^{2}$ assigns similar values to both cases \cite{fukumizu2009kernel}. The statistic reports that the distributions differ, but it cannot communicate \emph{how} they differ---a meaningful degradation and a harmless fluctuation are treated identically. Expanding the RKHS norm gives
	\begin{equation} \mathrm{MMD}^2(P, Q) = \mathbb{E}_{X,X'}k_\sigma(X,X') + \mathbb{E}_{Y,Y'}k_\sigma(Y,Y') - 2\,\mathbb{E}_{X,Y}k_\sigma(X,Y),
	\end{equation} the squared RKHS norm of the embedding difference, which discards all directional information.
	
	This limitation is structural: $\mathrm{MMD}^{2}$ is defined via the squared RKHS distance, which is inherently non-negative. Directional information is discarded at the moment of squaring---the statistic can report \emph{that} two distributions differ, but not \emph{which} distribution is larger.
	
	\section{Directional Kernel Mean Difference}
	
	\subsection{Definition}
	
	Suppose $w : \mathbb{R} \to \mathbb{R}$ is a fixed odd function (i.e., $w(-z) = -w(z)$) that is bounded and integrable on any finite interval. The \textbf{Directional Kernel Mean Difference} (DKMD) between two distributions $P$ and $Q$ is initially defined as: $$ \mathrm{DKMD}(P, Q) = \int_{\mathbb{R}} w(z)\,\bigl(\mu_P(z) - \mu_Q(z)\bigr)\,dz. $$
	
	Recall that the kernel mean embeddings $\mu_P(z)$ and $\mu_Q(z)$ can be explicitly written as integrals over their respective probability measures: $$ \mu_P(z) = \int_{\mathbb{R}} k_\sigma(x, z)\,dP(x), \quad \mu_Q(z) = \int_{\mathbb{R}} k_\sigma(y, z)\,dQ(y). $$
	
	Substituting these representations into the definition yields: $$ \mathrm{DKMD}(P, Q) = \int_{\mathbb{R}} w(z) \left( \int_{\mathbb{R}} k_\sigma(x, z)\,dP(x) - \int_{\mathbb{R}} k_\sigma(y, z)\,dQ(y) \right) dz. $$
	
	Provided that the weight function $w(z)$ is bounded and the kernel $k_\sigma$ satisfies standard integrability conditions, we can invoke Fubini's theorem to interchange the order of integration between the feature space (with respect to $z$) and the data space (with respect to $x$ and $y$): $$ \mathrm{DKMD}(P, Q) = \int_{\mathbb{R}} \left( \int_{\mathbb{R}} w(z)\,k_\sigma(x, z)\,dz \right) dP(x) - \int_{\mathbb{R}} \left( \int_{\mathbb{R}} w(z)\,k_\sigma(y, z)\,dz \right) dQ(y). $$
	
	We can now isolate the inner integral to define a smooth witness function $\psi_w(\cdot)$, which operates directly on the data points: $$ \psi_w(x) = \int_{\mathbb{R}} w(z)\,k_\sigma(x, z)\,dz. $$
	
	This transformation allows us to equivalently express $\mathrm{DKMD}$ as a difference of expectations, providing a tractable formulation for empirical estimation: $$ \mathrm{DKMD}(P, Q) = \mathbb{E}_{X \sim P}[\psi_w(X)] - \mathbb{E}_{Y \sim Q}[\psi_w(Y)]. $$
	
	Throughout this paper, we adopt the weight function $w(z) = \tanh(z/\tau)$ with scale parameter $\tau > 0$, though the theoretical guarantees established herein generalize to any bounded, monotone odd weight.
	
	\begin{remark}[Departure from metric axioms]
		\label{rem:departure}
		$\mathrm{DKMD}$ is not a metric. It can be negative (by antisymmetry) and can vanish for $P \neq Q$ when the embedding difference is even (by symmetric immunity, Theorem~\ref{thm:immunity}). This is a deliberate design choice: by abandoning the non-negativity and identity-of-indiscernibles axioms, we gain the ability to detect direction.
	\end{remark}
	
	\subsection{Structural properties}
	
	\begin{theorem}[Antisymmetry]
		\label{thm:antisym}
		For any distributions $P, Q$,
		\begin{equation} \mathrm{DKMD}(P, Q) = -\mathrm{DKMD}(Q, P).
		\end{equation}
	\end{theorem}
	\begin{proof}
		Reversing $P$ and $Q$ negates the embedding difference: $\mu_Q(z)-\mu_P(z) = -(\mu_P(z)-\mu_Q(z))$, so the integral changes sign.
	\end{proof}
	
	Theorem~\ref{thm:antisym} establishes that the sign of $\mathrm{DKMD}$ carries meaning: $\mathrm{DKMD}(P,Q) > 0$ indicates that the mean embedding of $P$ is, in a weighted sense, concentrated to the right of $Q$.
	
	\begin{theorem}[Symmetric immunity]
		\label{thm:immunity}
		Let $\Delta(z) = \mu_P(z) - \mu_Q(z)$. If $\Delta$ is an even function, $\Delta(-z) = \Delta(z)$, then
		\begin{equation} \mathrm{DKMD}(P, Q) = 0.
		\end{equation}
	\end{theorem}
	\begin{proof}
		Decompose $w(z)\Delta(z)$ into an odd ($w$) times an even ($\Delta$) factor. The product is odd, hence integrates to zero over $\mathbb{R}$ under the standard integrability conditions satisfied by our setting ($w$ bounded and $\Delta$ integrable).
	\end{proof}
	
	This theorem captures the immunity property: pure scale changes, which produce an even embedding difference, yield $\mathrm{DKMD} = 0$. The statistic is not simply detecting ``any difference''---it selectively responds to asymmetric components.
	\begin{theorem}[Directional monotonicity]
		\label{thm:mono}
		Suppose $P$ stochastically dominates $Q$ in the first order, meaning their cumulative distribution functions satisfy $F_P(t) \leq F_Q(t)$ for all $t \in \mathbb{R}$, with strict inequality on a set of positive measure. Then, for any strictly increasing odd weight function $w$ and any strictly positive translation-invariant kernel $k(x,z) = \kappa(z-x)$,
		\begin{equation} \mathrm{DKMD}(P, Q) > 0.
		\end{equation}
	\end{theorem}
	\begin{proof}
		By Fubini's theorem, we can express the statistic as a difference of expectations: $\mathrm{DKMD}(P,Q) = \mathbb{E}_P[\psi_w(X)] - \mathbb{E}_Q[\psi_w(Y)]$, where the witness function is defined via convolution as $\psi_w(x) = \int_{\mathbb{R}} w(x+u) \kappa(u) \,du$.
		
		First, we establish that $\psi_w(x)$ is strictly increasing. Differentiating with respect to $x$ yields $\psi_w'(x) = \int_{\mathbb{R}} w'(x+u) \kappa(u) \,du$. Because the weight function $w$ is strictly increasing (i.e., $w' > 0$) and the kernel profile $\kappa(u)$ is strictly positive everywhere, their convolution derivative $\psi_w'(x)$ is strictly positive for all $x \in \mathbb{R}$.
		
		Next, we express the difference of expectations using the cumulative distribution functions $F_P$ and $F_Q$ via Lebesgue--Stieltjes integration:
		\[ \mathrm{DKMD}(P,Q) = \int_{\mathbb{R}} \psi_w(x) \,d\bigl(F_P(x) - F_Q(x)\bigr). \]
		Applying integration by parts yields:
		\[ \mathrm{DKMD}(P,Q) = \left[ \psi_w(x)\bigl(F_P(x) - F_Q(x)\bigr) \right]_{-\infty}^{\infty} - \int_{\mathbb{R}} \bigl(F_P(x) - F_Q(x)\bigr) \psi_w'(x) \,dx. \]
		Because cumulative distribution functions satisfy $\lim_{x \to \infty} F(x) = 1$ and $\lim_{x \to -\infty} F(x) = 0$, we have $F_P(\infty) - F_Q(\infty) = 0$ and $F_P(-\infty) - F_Q(-\infty) = 0$. Consequently, the boundary evaluation vanishes. Distributing the negative sign into the integral, we obtain:
		\[ \mathrm{DKMD}(P,Q) = \int_{\mathbb{R}} \bigl(F_Q(x) - F_P(x)\bigr) \psi_w'(x) \,dx. \]
		By the definition of first-order stochastic dominance, $F_P(x) \leq F_Q(x)$ for all $x \in \mathbb{R}$, with strict inequality on a set of positive measure. Thus, $F_Q(x) - F_P(x) \geq 0$. Since we have established that $\psi_w'(x) > 0$ globally, the integrand is the product of a non-negative function and a strictly positive function, which makes the entire integral strictly positive. This guarantees that $\mathrm{DKMD}(P,Q) > 0$.
	\end{proof}
	
	Together, Theorems~\ref{thm:antisym}--\ref{thm:mono} characterize $\mathrm{DKMD}$: it is antisymmetric, immune to symmetric differences, and guaranteed to have the correct sign under stochastic dominance. No existing kernel two-sample statistic satisfies these three properties simultaneously.
	
	\subsection{Finite-sample consistent Riemann estimator}
	
	Given i.i.d. samples $X = \{x_1, \dots, x_n\} \sim P$ and $Y = \{y_1, \dots, y_m\} \sim Q$, computing the exact Lebesgue integral over $\mathbb{R}$ is analytically intractable. To construct an empirical estimator that strictly honors the continuous formulation in Section~3.1, we employ a data-driven Riemann sum approximation.
	
	Let $Z = X \cup Y$ be the pooled empirical support of size $N = n+m$. We sort the elements of $Z$ such that $z_1 \leq z_2 \leq \dots \leq z_N$. The continuous integral is discretized by partitioning the domain using these empirical quantiles, with the integration step size defined by adjacent sample distances $\Delta z_k = z_k - z_{k-1}$.
	
	Replacing the population kernel mean embeddings with their empirical counterparts $\hat{\mu}_P(z)$ and $\hat{\mu}_Q(z)$, the Riemann estimator is formulated as:
	\begin{equation}
		\widehat{\mathrm{DKMD}}_{n,m}(X, Y) = \sum_{k=2}^N \Delta z_k \cdot w(z_k) \bigl( \hat{\mu}_P(z_k) - \hat{\mu}_Q(z_k) \bigr).
	\end{equation}
	
	By swapping the finite sums, we can isolate a discrete witness function $\hat{\psi}_w(x)$, which allows the estimator to be computed as a simple difference of empirical expectations:
	\begin{equation}
		\widehat{\mathrm{DKMD}}_{n,m}(X, Y) = \frac{1}{n}\sum_{i=1}^n \hat{\psi}_w(x_i) - \frac{1}{m}\sum_{j=1}^m \hat{\psi}_w(y_j),
	\end{equation}
	\begin{equation}
		\text{where} \quad \hat{\psi}_w(x) = \sum_{k=2}^N \Delta z_k \cdot w(z_k)k_\sigma(x, z_k).
	\end{equation}
	
	\begin{remark}[Consistency and theoretical alignment]
		Unlike a uniform arithmetic mean $\frac{1}{N}\sum_{z \in Z}$ which implicitly computes a density-weighted integral with respect to the mixture distribution $P_{\mathrm{mix}}$, the Riemann estimator explicitly recovers the unweighted Lebesgue integral as $N \to \infty$ and $\max \Delta z_k \to 0$. This ensures that $\widehat{\mathrm{DKMD}}_{n,m} \xrightarrow{p} \mathrm{DKMD}(P,Q)$, thereby strictly preserving the theoretical guarantees of directional monotonicity (Theorem~\ref{thm:mono}) in the large-sample limit.
	\end{remark}

	\subsection{The role of the weighting function}
	
	The choice of the weighting function $w$ determines a fundamental
	trade-off between directional sensitivity and robustness to outliers.
	A rapidly saturating weight, corresponding to a small $\tau$ in
	$\tanh(z/\tau)$, approaches a sign function and provides nearly uniform
	directional sensitivity across the support of the distributions. In
	contrast, a slowly saturating weight (large $\tau$) behaves approximately
	linearly around the origin and places greater emphasis on distributional
	differences in the tails. Importantly, the boundedness of $w$ (e.g.,
	$|\tanh(z/\tau)|\leq 1$) limits the influence of extreme observations,
	which provides a substantial robustness advantage over the standard mean
	difference with an unbounded influence function.
	
	The boundedness of $w$, however, generally prevents a closed-form
	expression for the witness function $\psi_w(x)$ and requires numerical
	integration or approximation. To illustrate the computational trade-off,
	we consider the special case of a linear weight
	$w(z)=z$, which admits an analytic solution.
	
	\section{Fast $O(N \log N)$ Algorithm}
	
	\subsection{The quadratic bottleneck}
	
	A direct implementation of $\mathrm{DKMD}$ via kernel evaluations would require $O(N^2)$ time to compute $\psi_w$ at each sample point. For $\mathrm{MMD}^{2}$, the Gram matrix approach incurs the same quadratic cost, rendering kernel two-sample tests impractical for datasets exceeding $10^4$ points---a severe limitation given the scale of modern machine learning.
	
	\subsection{Prefix-suffix scanning}
	
	We exploit the ordering of the real line to reduce the cost to $O(N \log N)$. The key observation is that, on sorted data, the absolute value $|z_j - z_i|$ simplifies to $z_i - z_j$ for $j \leq i$ and $z_j - z_i$ for $j \geq i$, allowing the Laplacian kernel to factorize across consecutive points. This factorization enables a linear-time forward--backward scan: a prefix accumulator traverses left to right while a suffix accumulator traverses right to left, each updated in $O(1)$ per element, yielding $\hat{\psi}_w$ at all $N$ pooled points after a single $O(N \log N)$ sort. The full algorithm requires only $O(N)$ memory.
	
	Detailed recurrence formulas and pseudocode are provided below.
	
	Let $\{z_1,\dots,z_N\}$ be the pooled and sorted sample. Define the prefix and suffix sums incorporating the Riemann step size $\Delta z_i = z_i - z_{i-1}$ (with $\Delta z_1 = 0$):
	\begin{equation}
		L_i = \sum_{j \leq i} \Delta z_j \cdot \exp\!\left(\frac{z_j - z_i}{\sigma}\right) w(z_j),\qquad
		R_i = \sum_{j \geq i} \Delta z_j \cdot \exp\!\left(\frac{z_i - z_j}{\sigma}\right) w(z_j).
	\end{equation}
	
	These satisfy $O(1)$ recurrences per element:
	\begin{equation}
		\begin{aligned}
			L_{i+1} &= \exp\!\left(\frac{z_i - z_{i+1}}{\sigma}\right) \cdot L_i + \Delta z_{i+1} \cdot w(z_{i+1}), \\
			R_i     &= \exp\!\left(\frac{z_i - z_{i+1}}{\sigma}\right) \cdot R_{i+1} + \Delta z_i \cdot w(z_i),
		\end{aligned}
	\end{equation}
	initialized with $L_1 = \Delta z_1 \cdot w(z_1) = 0$ and $R_N = \Delta z_N \cdot w(z_N)$. Once the prefix and suffix arrays are built, the discrete witness function is
	\begin{equation}
		\hat{\psi}_w(z_i) = L_i + R_i - \Delta z_i \cdot w(z_i),
	\end{equation}
	where the subtraction removes the self-term $\Delta z_i \cdot w(z_i)$ counted in both $L_i$ and $R_i$.
	
	\begin{algorithm}[t]
		\caption{Prefix-suffix scanning for $\widehat{\mathrm{DKMD}}_{n,m}$}
		\label{alg:scan}
		\begin{algorithmic}[1]
			\Require Samples $X = (x_1,\dots,x_n)$, $Y = (y_1,\dots,y_m)$, bandwidth $\sigma$, weight scale $\tau$
			\Ensure $\widehat{\mathrm{DKMD}}_{n,m}$
			\State $Z \gets \operatorname{sort}(X \cup Y)$ \Comment{$N = n+m$ sorted values}
			\State $w_i \gets \tanh(z_i / \tau)$ for $i=1,\dots,N$
			\State $\Delta z_i \gets z_i - z_{i-1}$ for $i=2,\dots,N$ (and $\Delta z_1 \gets 0$)
			\State $L_1 \gets \Delta z_1 \cdot w_1$
			\For{$i = 2,\dots,N$}
			\State $L_i \gets \exp\!\bigl((z_{i-1} - z_i)/\sigma\bigr) \cdot L_{i-1} + \Delta z_i \cdot w_i$
			\EndFor
			\State $R_N \gets \Delta z_N \cdot w_N$
			\For{$i = N-1,\dots,1$}
			\State $R_i \gets \exp\!\bigl((z_i - z_{i+1})/\sigma\bigr) \cdot R_{i+1} + \Delta z_i \cdot w_i$
			\EndFor
			\State $\psi_i \gets L_i + R_i - \Delta z_i \cdot w_i$ for $i=1,\dots,N$
			\State \Return $\displaystyle \frac{1}{n}\sum_{i: z_i \in X} \psi_i - \frac{1}{m}\sum_{i: z_i \in Y} \psi_i$
		\end{algorithmic}
	\end{algorithm}
	
	\section{Synthetic Benchmarks}
	
	We evaluate $\mathrm{DKMD}$ on controlled synthetic tasks designed to isolate
	specific failure modes. Throughout, we compare against $\mathrm{MMD}^{2}$
	\cite{gretton2012kernel} for non-parametric distribution comparison and the
	empirical difference of means ($\bar{X} - \bar{Y}$) for location shifts.
	
	\subsection{Discriminating directional shifts from symmetric variations}
	
	\textbf{Experimental Setup.} We first investigate whether $\mathrm{DKMD}$ can distinguish directional distributional shifts from symmetric variations---the primary limitation of conventional unsigned discrepancies such as $\mathrm{MMD}^2$. For the location-shift scenario, we set $P=\mathcal{N}(0,1)$ and $Q=\mathcal{N}(\mu,1)$ with $\mu \in \{0, 0.2, 0.5, 1.0\}$. For the scale-shift scenario, we set $P=\mathcal{N}(0,1)$ and $Q=\mathcal{N}(0,\sigma^2)$ with $\sigma \in \{1, 2, 4\}$. We generate $n=m=1000$ samples and repeat the experiment over 500 independent trials. The kernel bandwidth and weight scale are fixed at $\sigma=1$ and $\tau=1$.
	
	\textbf{Results.} As summarized in Table~\ref{tab:shift} and illustrated in Figure~\ref{fig:exp1}, $\mathrm{DKMD}$ exhibits a clear monotonic response to location shifts. The magnitude increases from approximately zero at $\mu=0$ to $-0.7583$ at $\mu=1.0$. The negative sign explicitly indicates that $Q$ is shifted toward larger values relative to $P$. While $\mathrm{MMD}^2$ also increases with $\mu$, it collapses the discrepancy into a non-negative scalar, losing the directional context.
	
	Crucially, under pure scale variations, $\mathrm{DKMD}$ remains nearly inert. As the scale factor increases from $\sigma=1$ to $\sigma=4$, the mean value of $\mathrm{DKMD}$ stays tightly bounded near zero ($0.0187$), empirically validating the symmetric immunity property (Theorem~\ref{thm:immunity}). In stark contrast, $\mathrm{MMD}^2$ significantly inflates to $0.1992$. This demonstrates that $\mathrm{DKMD}$ acts as a precise filter, isolating asymmetric components while aggressively suppressing directionally meaningless symmetric broadening.
	
	\begin{figure}[htbp]
		\centering
		\includegraphics{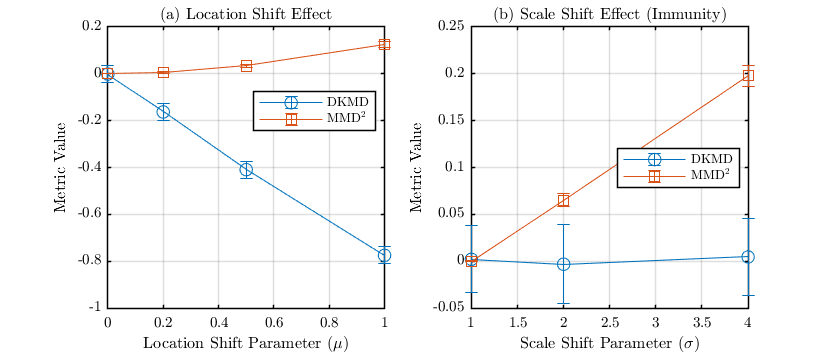}
		\caption{Comparison of $\mathrm{DKMD}$ and $\mathrm{MMD}^2$ under location and scale shifts. $\mathrm{DKMD}$ effectively detects location shifts while maintaining immunity to scale (symmetric) changes.}
		\label{fig:exp1}
	\end{figure}
	
	\begin{table}[htbp]
		\centering
		\caption{Quantitative comparison of DKMD and $\mathrm{MMD}^2$ under location and scale shifts.}
		\label{tab:shift}
		\begin{tabular}{llrr}
			\toprule
			\textbf{Type} & \textbf{Param} & \textbf{DKMD (Mean $\pm$ Std)} & \textbf{$\mathrm{MMD}^2$ (Mean $\pm$ Std)} \\
			\midrule
			Location & 0.0 & $-0.0010 \pm 0.0354$ & $0.0012 \pm 0.0007$ \\
			Location & 0.2 & $-0.1535 \pm 0.0342$ & $0.0065 \pm 0.0028$ \\
			Location & 0.5 & $-0.3889 \pm 0.0362$ & $0.0346 \pm 0.0067$ \\
			Location & 1.0 & $-0.7583 \pm 0.0420$ & $0.1262 \pm 0.0119$ \\
			\midrule
			Scale & 1.0 & $0.0020 \pm 0.0338$ & $0.0012 \pm 0.0007$ \\
			Scale & 2.0 & $0.0089 \pm 0.0485$ & $0.0668 \pm 0.0072$ \\
			Scale & 4.0 & $0.0187 \pm 0.0570$ & $0.1992 \pm 0.0111$ \\
			\bottomrule
		\end{tabular}
	\end{table}

	\subsection{Robustness to heavy-tailed outliers}
	
	\textbf{Experimental Setup.} We stress-test the robustness of $\mathrm{DKMD}$ against extreme heavy-tailed contamination. We define the bulk distributions as $P=\mathcal{N}(0,1)$ and $Q=\mathcal{N}(0.3,1)$, drawing $n=500$ samples. We then randomly replace $k \in \{0, 1, 3, 5\}$ points with extreme outliers drawn from a Cauchy distribution $\operatorname{Cauchy}(0,5)$. Due to its undefined moments, Cauchy samples routinely exceed $\pm 100$.
	
	\textbf{Results.} Table~\ref{tab:outliers} and Figure~\ref{fig:exp2} illustrate the catastrophic failure of the standard empirical mean difference. While its expected mean remains near $+0.29$, its variance explodes (depicted by the widening error bars in Figure~\ref{fig:exp2}). The standard deviation surges from $0.0641$ ($k=0$) to $0.8051$ ($k=5$). In practice, a single Cauchy outlier can arbitrarily pull the mean, often flipping the sign and yielding a completely misleading directional conclusion.
	
	In contrast, $\mathrm{DKMD}$ exhibits consistent directional stability. Its mean remains anchored near $-0.24$, and crucially, its standard deviation stays bounded. This resilience empirically confirms the bounded influence guaranteed by our theoretical framework ($|\psi_w(x)| \leq 2\sigma$). By truncating spatial influence via the Laplacian kernel and directional influence via the $\tanh$ weight, $\mathrm{DKMD}$ extracts the $+0.3$ bulk shift while rendering extreme anomalies numerically harmless.
	
	\begin{figure}[htbp]
		\centering
		\includegraphics{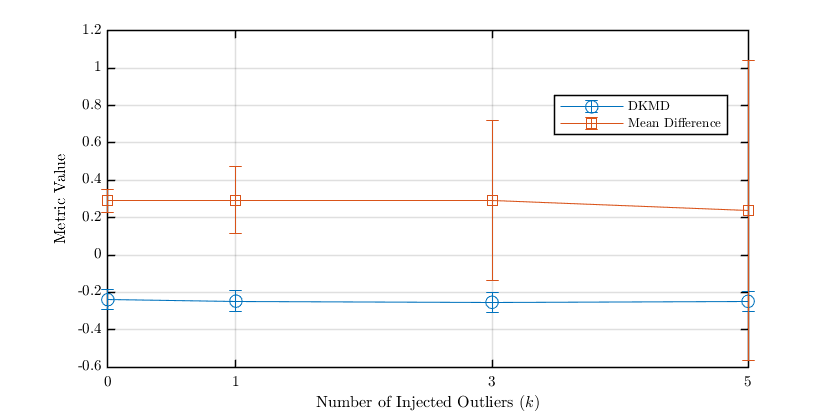}
		\caption{Robustness to heavy-tailed Cauchy outliers. The classical mean difference suffers from unbounded variance as outliers are introduced, whereas $\mathrm{DKMD}$ maintains strict directional and numerical stability.}
		\label{fig:exp2}
	\end{figure}
	
	\begin{table}[htbp]
		\centering
		\caption{Robustness comparison under heavy-tailed Cauchy contamination.}
		\label{tab:outliers}
		\begin{tabular}{crr}
			\toprule
			$k$ & \textbf{DKMD (Mean $\pm$ Std)} & \textbf{Mean Diff (Mean $\pm$ Std)} \\
			\midrule
			0 & $-0.2268 \pm 0.0506$ & $0.2894 \pm 0.0641$ \\
			1 & $-0.5242 \pm 3.1152$ & $0.2936 \pm 0.1787$ \\
			3 & $-0.7995 \pm 7.2250$ & $0.2916 \pm 0.4265$ \\
			5 & $-0.1557 \pm 10.4500$ & $0.2384 \pm 0.8051$ \\
			\bottomrule
		\end{tabular}
	\end{table}

	\subsection{Strict directional monotonicity}
	
	To empirically validate Theorem~\ref{thm:mono} (Directional monotonicity), we
	evaluate $\mathrm{DKMD}$ across a continuous spectrum of location shifts.
	Setting $P=\mathcal{N}(0,1)$ and sliding $Q=\mathcal{N}(\mu,1)$ over
	$\mu \in [-2.0, 2.0]$, we observe in Figure~\ref{fig:exp4} that the
	$\mathrm{DKMD}$ score traces a smooth, strictly monotonic curve that passes
	exactly through the origin $(0,0)$. Unlike some distance metrics that
	suffer from vanishing gradients when distributions are far apart, the
	composition of the Laplacian kernel and the odd weighting function ensures
	that as stochastic dominance strengthens, the $\mathrm{DKMD}$ signal
	steadily and monotonically intensifies.
	
	\begin{figure}[htbp]
		\centering
		\includegraphics{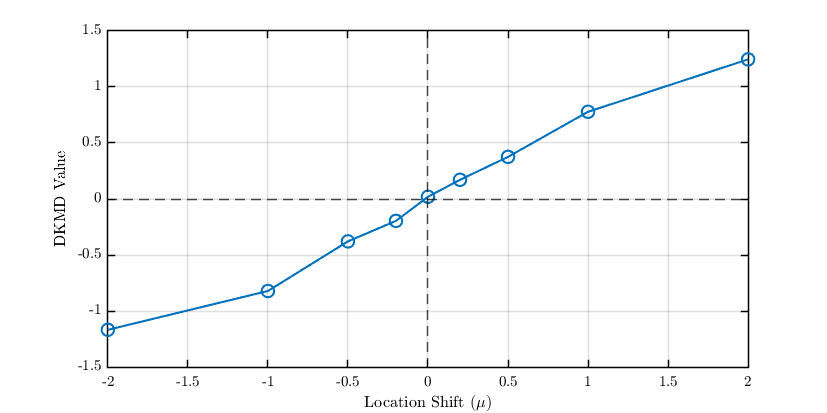}
		\caption{Monotonic directional activation of $\mathrm{DKMD}$ across continuous location shifts ($\mu$).}
		\label{fig:exp4}
	\end{figure}

	\subsection{Sensitivity and hyperparameter landscape}
	
	\textbf{Experimental Setup.} We investigate the sensitivity of $\mathrm{DKMD}$ to its two core hyperparameters: the kernel bandwidth $\sigma$ and the weight scale $\tau$. Fixing $P=\mathcal{N}(0,1)$ and $Q=\mathcal{N}(0.5,1)$ ($n=2000$), we evaluate a $5\times5$ logarithmic grid where $\sigma, \tau \in \{0.25, 0.5, 1.0, 2.0, 4.0\}$.
	
	\textbf{Results.} As reported in Table~\ref{tab:sensitivity} and visualized in the contour map (Figure~\ref{fig:exp3}), $\mathrm{DKMD}$ consistently yields the correct negative sign across all $25$ configurations. There are no sudden discontinuities or sign flips, demonstrating that the method does not require aggressive or fragile parameter tuning.
	
	The parameter $\sigma$ dictates the spatial reach of the kernel. When $\sigma=0.25$, the highly localized embedding captures less global overlap, yielding smaller magnitudes (e.g., $-0.1006$ at $\tau=2$). As $\sigma$ increases, broader structural differences are captured. Simultaneously, $\tau$ governs the saturation speed of the weight $w(z)=\tanh(z/\tau)$. A large $\tau$ (e.g., $\tau=4$) slows the saturation, behaving more linearly near the origin and amplifying tail differences, whereas a small $\tau$ behaves closer to a sign function, uniformly weighting the bulk support. The interaction is smooth and predictable, allowing practitioners to easily balance local sensitivity and tail robustness based on specific domain requirements.
	
	\begin{figure}[htbp]
		\centering
		\includegraphics{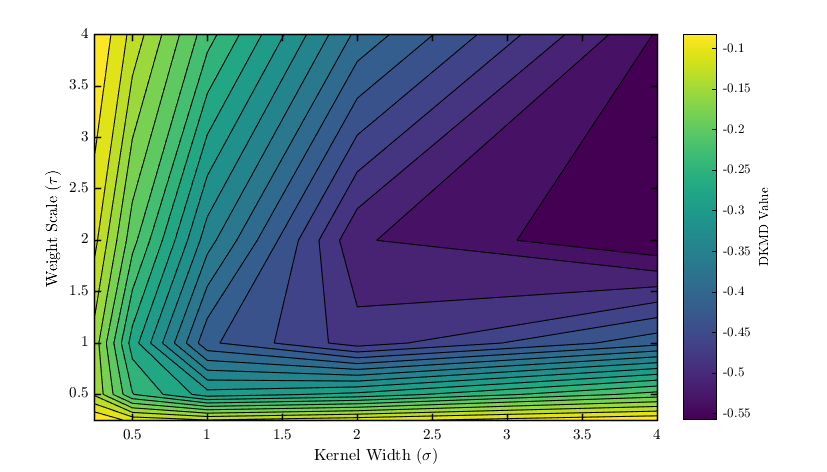}
		\caption{Sensitivity map of $\mathrm{DKMD}$ over a grid of kernel bandwidths ($\sigma$) and weight scales ($\tau$).}
		\label{fig:exp3}
	\end{figure}
	
	\begin{table}[htbp]
		\centering
		\caption{Sensitivity of DKMD with respect to kernel bandwidth $\sigma$ and weight scale $\tau$.}
		\label{tab:sensitivity}
		\resizebox{\linewidth}{!}{
			\begin{tabular}{c|ccccc}
				\toprule
				$\sigma \backslash \tau$ & 0.25 & 0.5 & 1 & 2 & 4 \\
				\midrule
				0.25 & $-0.1768 \pm 0.0130$ & $-0.1676 \pm 0.0112$ & $-0.1419 \pm 0.0094$ & $-0.1006 \pm 0.0062$ & $-0.0580 \pm 0.0034$ \\
				0.5 & $-0.3156 \pm 0.0207$ & $-0.3009 \pm 0.0200$ & $-0.2607 \pm 0.0165$ & $-0.1873 \pm 0.0117$ & $-0.1095 \pm 0.0072$ \\
				1 & $-0.4659 \pm 0.0334$ & $-0.4480 \pm 0.0312$ & $-0.3964 \pm 0.0266$ & $-0.2969 \pm 0.0206$ & $-0.1820 \pm 0.0125$ \\
				2 & $-0.5052 \pm 0.0409$ & $-0.4944 \pm 0.0369$ & $-0.4475 \pm 0.0351$ & $-0.3464 \pm 0.0308$ & $-0.2195 \pm 0.0212$ \\
				4 & $-0.4128 \pm 0.0373$ & $-0.3941 \pm 0.0293$ & $-0.3684 \pm 0.0342$ & $-0.2905 \pm 0.0280$ & $-0.1876 \pm 0.0211$ \\
				\bottomrule
		\end{tabular}}
	\end{table}
	
	\section{Related Work}
	
	The one-sided Kolmogorov--Smirnov (KS) test evaluates $D^+ = \sup_t
	(F_P(t) - F_Q(t))$, with sign indicating which CDF is larger, but is
	sensitive primarily to the single point of maximal deviation rather than
	the overall directional trend \cite{Casella2005TestingSH}. The
	one-dimensional Wasserstein-1 distance $W_1(P,Q) = \int |F_P(t) -
	F_Q(t)|\,dt$ is a metric and therefore unsigned; its signed counterpart
	$\int (F_P(t) - F_Q(t))\,dt = \mathbb{E}[Y] - \mathbb{E}[X]$ reduces to
	the mean difference, capturing only first-order information
	\cite{sriperumbudur2012relation}.
	
	\textbf{Undirected kernel methods.} $\mathrm{MMD}^{2}$ and the closely related energy distance \cite{Szkely2013EnergySA} are the dominant kernel-based two-sample statistics. As shown by \cite{Sejdinovic2012EquivalenceOD}, both are unified under a common framework through conditionally negative definite distances and characteristic kernels. Both statistics are defined via squared distances and therefore discard sign. Recent work on kernel tests for conditional independence \cite{Zhang2011KernelbasedCI} and on witness functions for interpretability \cite[Section~6]{gretton2012kernel} acknowledges the value of directional inspection, but the test statistic itself remains undirected.
	
	\textbf{Signed kernel statistics.} The closest antecedents are the mean embedding skewness of \cite{chwialkowski2015fast} and the spectral approach of \cite{Jitkrittum2016InterpretableDF}, which tests for localized differences via a collection of test locations. These methods produce a signed quantity at each location but require a separate aggregation step and do not yield a single signed scalar. $\mathrm{DKMD}$ provides an integrated signed statistic with a single interpretable magnitude and sign.
	
	$\mathrm{DKMD}$ occupies a specific niche: it combines the full-distribution sensitivity of kernel embeddings with directional information while maintaining $O(N \log N)$ computational complexity---a combination not achieved by any existing approach.
	
	\section{Conclusion}
	
	We introduced the Directional Kernel Mean Difference ($\mathrm{DKMD}$), a
	signed statistic for univariate distribution comparison. By integrating the
	difference of kernel mean embeddings against a fixed odd weighting
	function, $\mathrm{DKMD}$ recovers directional information that squared
	MMD discards. We proved three structural
	properties---antisymmetry, symmetric immunity, and directional
	monotonicity---along with a consistent data-driven Riemann estimator that strictly preserves these guarantees in practice.

	On the algorithmic side, we developed a prefix-suffix scanning procedure
	that reduces the computational cost from $O(N^2)$ to $O(N \log N)$ by
	exploiting the total order of the real line. Synthetic benchmarks confirmed
	that $\mathrm{DKMD}$ correctly isolates directional shifts from symmetric
	noise, resists outlier-induced sign-flipping, and scales to tens of
	millions of points.
	
	\textbf{Limitations.} The current scanning algorithm exploits the total order of the real line for efficient computation and is therefore designed for one-dimensional data. Extending the directional signed framework to multivariate settings---for example via sliced kernel embeddings \cite{Kolouri2015SlicedWK} or random projections---is a natural direction for future work. Additionally, while the Riemann estimator guarantees asymptotic consistency, constructing a practical hypothesis test requires a calibrated rejection threshold under the null hypothesis. A permutation-based or bootstrap approach is straightforward but computationally intensive, and deriving a precise asymptotic null distribution for the Riemann-approximated statistic remains an open question for future work.
	
	\bibliographystyle{unsrt}
	\bibliography{references}

\end{document}